\documentclass{article}
\usepackage{ijcai26}

\usepackage{times}
\usepackage{soul}
\usepackage{url}
\usepackage[hidelinks]{hyperref}
\usepackage[utf8]{inputenc}
\usepackage[small]{caption}
\usepackage{graphicx}
\usepackage{amsmath}
\usepackage{amsthm}
\usepackage{rotating}
\usepackage{float}
\usepackage{placeins}
\usepackage{amssymb}
\usepackage{booktabs}
\usepackage[linesnumbered,ruled,vlined]{algorithm2e}
\usepackage{algorithmic}
\usepackage[switch]{lineno}
\usepackage{multirow}

\usepackage{color}

\newtheorem{theorem}{Theorem}

\title{PerFlow: Physics-Embedded Rectified Flow for Efficient Reconstruction \\ and Uncertainty Quantification of Spatiotemporal Dynamics}

\author{
    Author Name
    \affiliations
    Affiliation
    \emails
    email@example.com
}

\author{
Hao Zhou$^1$
\and
Rui Zhang$^1$\thanks{Corresponding authors: Rui Zhang and Hao Sun.}\and
Han Wan$^{1}$\And
Hao Sun$^{1*}$\\
\affiliations
$^1$Gaoling School of Artificial Intelligence, Renmin University of China\\
\emails
\{haozhou, rayzhang, wanhan2001, haosun\}@ruc.edu.cn
}

\begin{document}

\maketitle

\begin{abstract}
Reconstructing PDE-governed fields from sparse and irregular measurements is challenging due to their ill-posed nature. Deterministic surrogates are trained on dense fields that struggle with limited measurements and uncertainty quantification. Generative models, by learning distributions over spatiotemporal fields, can better handle sparsity and uncertainty. However, existing generative approaches enforce data consistency and PDE constraints simultaneously via sampling-time gradient guidance, resulting in slow and unstable inference. To this end, we propose \textbf{PerFlow}, a \textbf{P}hysics-\textbf{e}mbedded \textbf{r}ectified \textbf{Flow} for efficient sparse reconstruction and uncertainty quantification of spatiotemporal dynamics. PerFlow decouples observation conditioning from physics enforcement, performing guidance-free conditioning by feeding observations into rectified-flow dynamics while embedding hard physics via a constraint-preserving projection (e.g., incompressibility or conservation). Theoretically, we establish invariance guarantees to ensure that trajectories remain on the physics-consistent manifold throughout sampling. Experiments on various PDE systems demonstrate competitive reconstruction accuracy with sound physics consistency, while enabling efficient conditional sampling (e.g., 50 steps) and up to 320x faster inference than 2000-step guided diffusion baselines.
\end{abstract}

\section{Introduction}

Partial differential equations (PDEs) govern various dynamical systems in science and engineering~\cite{holton2013introduction,tu2023computational,strogatz2024nonlinear}. While classical numerical solvers offer high accuracy and convergence guarantees, they require fine-grained spatiotemporal discretization, resulting in prohibitively high computational costs~\cite{lapidus1999numerical,tadmor2012review}. To address this issue, neural surrogate models learn the solution operator between function spaces, aiming for rapid inference after training~\cite{li2020fno,lu2021deeponet,zhang2025monte}. Most existing surrogate models are deterministic and trained on gridded inputs (e.g., complete initial conditions), yet real-world measurement data is often sparse and irregular, rendering field reconstruction an ill-posed and highly challenging problem~\cite{zhang2015sparse_survey}. Furthermore, deterministic methods only provide point estimates and lack uncertainty estimation capability.

Recent advances in generative models have enabled probabilistic surrogates for PDE systems. Diffusion~\cite{ho2020denoising,song2020score} and flow~\cite{lipman2022flow} models learn a probability distribution over spatiotemporal fields, where stochastic sampling naturally supports uncertainty quantification. For forward modeling, generative PDE surrogates can capture multi-scale spectral features and quantify uncertainty, improving robustness to noise and distribution shift~\cite{lippe2023pderefiner,shysheya2024conditional,li2025flowpde}. More importantly, for sparse reconstruction, they can support conditionally guided generation, offering multiple high-fidelity plausible fields consistent with limited measurements~\cite{huang2024diffusionpde,li2024s3gm}.

\begin{figure}
    \centering
    \includegraphics[width=1\linewidth]{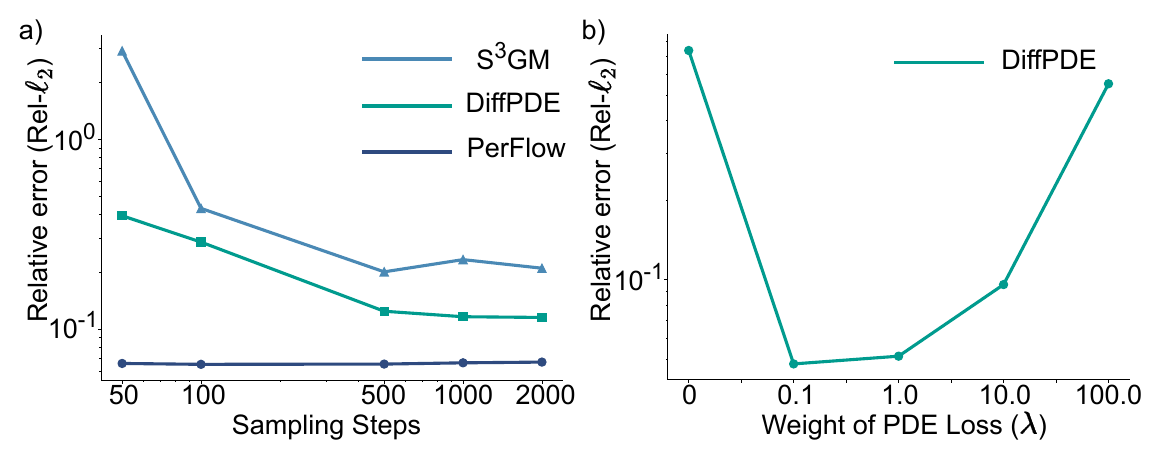}
    \caption{\textbf{(a)} Relative $\ell_2$ error vs.\ sampling steps on Darcy (S$^3$GM, DiffPDE, PerFlow). \textbf{(b)} DiffPDE sensitivity on Poisson: Relative $\ell_2$ error vs.\ PDE-loss weight $\lambda$.}
    \label{fig:poisson_weight}
\end{figure}

Many generative methods achieve controllable generation by enforcing data consistency and physics constraints during sampling via gradient-based guidance~\cite{li2024s3gm,wang2025fundiffdiffusionmodelsfunction}.
However, this introduces two major issues. Firstly, it requires calculating observation-consistency gradients and physics gradients at each step of the sampling process, resulting in the slow inference speeds~\cite{li2024s3gm}. Moreover, guidance tends to increase the required number of sampling steps and forces smaller step sizes for stable reconstructions. Secondly, it suffers from \emph{sensitivity to hyperparameters}, one must carefully balance data-consistency and physics-informed terms to conform the prior distribution and guidance information, where reconstruction quality is highly sensitive to these weights in practice~\cite{huang2024diffusionpde}. Fig.~\ref{fig:poisson_weight} illustrates these issues on two representative guided diffusion baselines (DiffPDE~\cite{huang2024diffusionpde} and S$^3$GM~\cite{li2024s3gm}).

To this end, we propose the \textbf{P}hysics-\textbf{e}mbedded \textbf{r}ectified \textbf{Flow} (\textbf{PerFlow}, shown in Fig.~\ref{fig:model}), an efficient conditional generative framework for sparse reconstruction and uncertainty quantification under partial observations. 
Unlike sampling-time guidance, PerFlow decouples observation conditioning from physics enforcement, avoiding gradient-based guidance at inference.
For observed information, PerFlow incorporates sparse measurements as conditions instead of conducting iterative data-fidelity gradients during sampling. 
For physics information, PerFlow encodes conservation laws through a constraint-preserving projection, which can ensure the entire generation remains on the physics-consistent manifold.
Finally, by adopting rectified flow, PerFlow supports fast few-step sampling, enabling efficient conditional sampling. 
In summary, we make the following contributions:

\noindent 1. We propose \textbf{PerFlow}, a physics-embedded rectified-flow framework for sparse-field reconstruction and uncertainty quantification, enabling guidance-free conditional sampling.

\noindent 2. We enforce hard physical constraints via constraint-preserving projection, and provide theoretical guarantees that sampling trajectories remain on the physics-consistent manifold throughout sampling.

\noindent 3. We evaluate the performance of PerFlow across various PDE benchmarks. PerFlow achieves competitive reconstruction accuracy with remarkably lower physics error, while providing $\sim$320$\times$ speedup over guided diffusion baselines.

\section{Related Work}

\subsection{Neural Surrogates for PDEs}
With the development of deep learning, neural surrogate models have emerged as a data-driven approach to accelerating PDE simulation~\cite{azizzadenesheli2024neural,zhang2024deciphering,bhaganagar2025accelerated}. The main idea of neural surrogates is using a neural network to learn the mapping between function space, such as from PDE initial conditions to the corresponding solution. Compared with classical numerical methods, neural surrogate models do not need fine-grained spatiotemporal discretization and thus achieve remarkable speedup. DeepONet and its variants~\cite{lu2021deeponet,HE2024117130} use branch and trunk networks to encode the physical fields and spatiotemporal coordinates, respectively. Fourier neural operator and its variants~\cite{li2020fno,rahman2022uno,xiong2024koopman,li2024pino,ijcai2025p862} conduct function transformation in the frequency domain to achieve discrete invariance. Moreover, graph-based models~\cite{eliasof2021gcn,zeng2024phympgn} and transformer-based models~\cite{li2023factformer,wu2024transolver} also extend neural surrogates to the irregular regime with their flexibility. However, most existing models use deterministic frameworks and conduct the training process across the complete fields, thus struggling with sparsity and uncertainty.

\subsection{Generative Modeling of PDE Systems}
Generative models offer a complementary approach to learning spatiotemporal systems by modeling \emph{distributions} over trajectories rather than producing a single point estimate, such as variational autoencoders~\cite{kingma2013vae}, normalizing flows~\cite{rezende2015normalizing_flow}, GANs~\cite{goodfellow2020gan}, and diffusion models~\cite{ho2020denoising,song2020ddim,lu2025dpmsolverplusplus} that sample via iterative denoising. Recently, continuous-time flow-based formulations, such as flow matching~\cite{lipman2022flow} and rectified flow~\cite{liu2022rectified}, learn a time-dependent velocity field and generate samples by integrating an ODE, reducing the number of function evaluations required for sampling. 
For physical systems, these models have been used for probabilistic forward simulation and for various downstream tasks~\cite{xu2022geodiff,du2024conditional,wang2025fundiffdiffusionmodelsfunction,hu2025wavelet}, including parameter inference, system control, and sparse reconstruction~\cite{wei2024diffphycon,li2024s3gm,huang2024diffusionpde,hu2025from}. A common paradigm is to model the joint distribution of the full spatiotemporal fields and incorporate measurements and physics information through sampling-time guidance~\cite{huang2024diffusionpde}. However, this framework requires iterative denoising with per-step backpropagation, leading to \emph{slow inference} and requiring more sampling steps or smaller step sizes for stable reconstructions. Moreover, these methods suffer from \emph{hyperparameter sensitivity}, since reconstruction quality depends strongly on the relative weights used to balance observation and physics terms.

\begin{figure*}[t!]
    \centering
    \includegraphics[width=\textwidth]{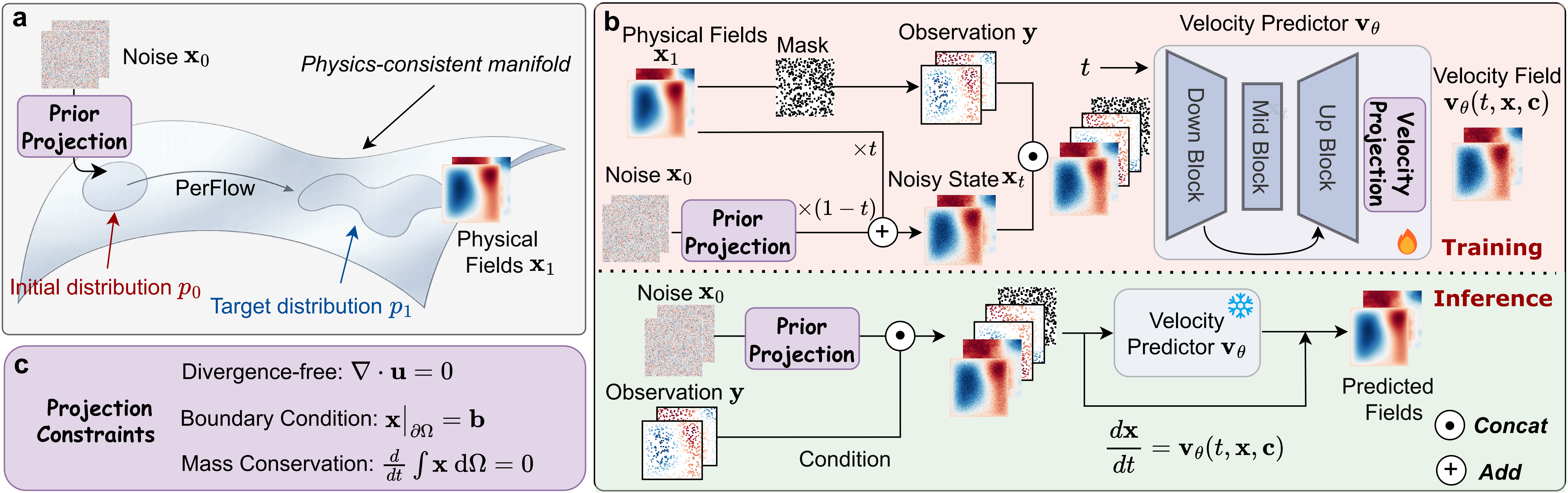}
    \caption{Overview of PerFlow.
    {
    \textbf{(a)} PerFlow learns a rectified-flow transport between initial and target distributions on the physics-consistent manifold.
    \textbf{(b)} Training and inference pipelines. During training, sparse observations are formed by masking ground-truth fields, while prior-projected noise is combined with the target field using weights $(1-t)$ and $t$ to construct $\mathbf{x}_t$. PerFlow learns a conditional velocity field $v_\theta(t,\mathbf{x},\mathbf{c})$, followed by velocity projection. During inference, sampled noise is prior-projected and evolved by integrating $\frac{d\mathbf{x}}{dt}=v_\theta(t,\mathbf{x},\mathbf{c})$. 
    \textbf{(c)} Projection constraints embedded in PerFlow, including divergence-free constraints, boundary conditions, and mass conservation.}
    }
\label{fig:model}
\end{figure*}

\section{Method}
\label{sec:method}

\subsection{Problem Setup}

We study sparse reconstruction and uncertainty quantification for PDE-governed physical fields. Let $\Omega \subset \mathbb{R}^d$ be a spatial domain and $[0,T]$ a time interval. A broad class of systems can be written as

\begin{equation}
    \mathcal{F}\left(\mathbf{x}; \boldsymbol{\theta}\right)=\mathbf{0},
    \quad \mathbf{x}\big|_{\partial\Omega}=\mathbf{b},
    \label{eq:pde_abstract}
\end{equation}
where $\mathbf{x}$ denotes the spatiotemporal field, $\boldsymbol{\theta}$ denotes PDE parameters, and $\mathbf{b}$ represents boundary conditions. In practice, only sparse and irregular measurements are available. We model sparse observations with a binary mask $\mathcal{M}$, where $\mathcal{M}_{i}=1$ indicates an observed entry and $\mathcal{M}_{i}=0$ otherwise:
\begin{equation}
    \mathbf{y} = \mathcal{M}\odot \mathbf{x}.
    \label{eq:obs_model}
\end{equation}
Starting from sparse measurements $\mathbf{y}$, PerFlow infers the complete field $\mathbf{x}$ and represents uncertainty through samples from the conditional distribution

\begin{equation}
p(\mathbf{x}\mid \mathbf{y},\mathcal{M}).
\label{eq:posterior}
\end{equation}

\subsection{Overview of PerFlow}\label{sec:overview}

Most posterior-guided diffusion/flow approaches enforce data consistency and physics constraints \emph{during sampling} via gradient-based guidance, resulting in slow inference and sensitivity to the guidance weights. To address these issues, PerFlow adopts rectified flow for efficient ODE-based sampling and decouples observation conditioning from physical guidance. PerFlow has the following key designs:

\noindent 1. \textbf{Rectified flow for few-step inference.}
Based on rectified flow, PerFlow samples by integrating a conditional ODE with a small number of function evaluations, enabling efficient reconstruction and uncertainty quantification.

\noindent 2. \textbf{Guidance-free conditioning.}
We regard sparse reconstruction as \emph{conditional generation} and explicitly provide the sparse measurements as conditions, avoiding sampling-time data-fidelity gradients.

\noindent 3. \textbf{Constraint-preserving projection.}
We use \emph{constraint-preserving projection} to ensure that trajectories remain on the physics-consistent manifold throughout sampling.

\subsection{Rectified Flow Preliminaries}
Rectified flow (RF)~\cite{liu2022rectified} learns a continuous-time transport that maps a simple base distribution to a target conditional data distribution.
Let $p_0$ be a base distribution and $p_1(\cdot \mid \mathbf{c})$ the target conditional distribution given $\mathbf{c}$.
RF learns a time-dependent velocity field predictor $\mathbf{v}_\theta(t,\mathbf{x},\mathbf{c})$ and generates samples by integrating
\begin{equation}
\frac{d\mathbf{x}}{dt}=\mathbf{v}_\theta(t,\mathbf{x},\mathbf{c}),\qquad t\in[0,1].
\label{eq:rf_ode}
\end{equation}

Training is based on flow matching along an interpolation path.
Given $\mathbf{x}_1\sim p_1(\cdot\mid \mathbf{c})$ and $\mathbf{x}_0\sim p_0$, we construct
\begin{equation}
\mathbf{x}_t = (1-t)\mathbf{x}_0 + t\mathbf{x}_1 ,
\qquad t\sim \mathcal{U}(0,1),
\label{eq:rf_path}
\end{equation}
and use the rectified-flow target velocity
\begin{equation}
\mathbf{v}^\star = \mathbf{x}_1-\mathbf{x}_0.
\label{eq:rf_target}
\end{equation}

Compared to diffusion-based sampling, RF admits efficient ODE solvers and typically requires fewer function evaluations, which is crucial for fast PDE reconstruction.

\subsection{Conditioning on Sparse Observations}
\label{sec:conditioning}
Unlike posterior-guided diffusion/flow methods that enforce observation consistency via sampling-time gradients, PerFlow casts sparse reconstruction as \emph{conditional generation}. The observation information is incorporated through the condition input, enabling guidance-free inference without per-step backpropagation.

Given sparse measurements $\mathbf{y}$ and mask $\mathcal{M}$, we use a mask-aware condition
\begin{equation}
\mathbf{c}= \mathrm{concat}\big(\mathbf{y},\,\mathcal{M}\big),
\label{eq:cond_def}
\end{equation}
where $\mathcal{M}$ is broadcast to match channels if needed.
We then learn a conditional velocity field $\mathbf{v}_\theta(t,\mathbf{x},\mathbf{c})$ parameterized by a U-Net backbone.
Training and inference procedures are detailed in Sec.~\ref{sec:train_sample}.

\subsection{Constraint-preserving Projection}
\label{sec:physics_manifold}
A central design of PerFlow is to restrict the rectified-flow dynamics to a \emph{physics-consistent manifold} that encodes hard PDE priors (Fig.~\ref{fig:model}\textbf{a}). Instead of injecting physics through gradient guidance, we propose the constraint-preserving projection to enforce essential constraints into the model architecture and sampling pipeline so that generated fields satisfy them by construction.

\paragraph{Where physics is embedded.}

PerFlow enforces hard constraints at two points:
{(i) the \emph{physics-embedded prior}, which is obtained by applying a \emph{prior projection} to the sampled noise, thereby yielding a physics-feasible initial state $\mathbf{x}(0)$;}
and
(ii) the \emph{constraint-preserving velocity projection}, where the learned velocity field is mapped to satisfy the corresponding constraints required by each ODE updates.
We implement these designs using a set of physics-embedded operators, 
{whose underlying physical constraints are summarized in Fig.~\ref{fig:model}\textbf{c} for clarity.}

\paragraph{(1) Divergence-free.}

For 2D incompressible flow, we enforce incompressibility with a stream-function construction, where we represent the velocity field $\mathbf{u}=(u,v)$ with a scalar stream function $\psi$:
\begin{equation}
u=\partial_y\psi,\qquad v=-\partial_x\psi,
\label{eq:streamfunction}
\end{equation}
which satisfies $\nabla\!\cdot\!\mathbf{u}=0$ by construction. For the \emph{physics-embedded prior}, we sample $\psi$ from a Gaussian random field, and project it to a divergence-free initial state $\mathbf{u}(0)=\nabla^\perp \psi$. For the \emph{constraint-preserving velocity projection}, the network predicts $\psi$, and we map it to $(u,v)$ to obtain divergence-free velocity updates.

\paragraph{(2) Dirichlet boundary conditions.}
Let $\mathcal{B}$ be a binary boundary indicator (1 on $\partial\Omega$, 0 otherwise) and $\mathbf{b}$ represents boundary conditions.
In detail, we replace the boundary values and keep interior entries unchanged to encode Dirichlet boundary conditions, which can be achieved via the operator $\Pi_{\mathrm{bc}}$ as follows:
\begin{equation}
\Pi_{\mathrm{bc}}(\tilde{\mathbf{x}})=\mathcal{B}\odot \mathbf{b} + (1-\mathcal{B})\odot \tilde{\mathbf{x}}.
\label{eq:dirichlet}
\end{equation}
During ODE integration, we also keep boundary values fixed by applying the same boundary mask to the velocity, i.e., we set boundary updates to zero using $(1-\mathcal{B})\odot(\cdot)$.

\paragraph{(3) Global conservation.}
For scalar conserved quantities, we enforce a fixed total mass using the mean-correction operator $\Pi_{\mathrm{mass}}$, which shifts the field to match the target mass:
\begin{equation}
\Pi_{\mathrm{mass}}(\tilde{\mathbf{x}})
=
\tilde{\mathbf{x}}
-\mathrm{mean}_{\Omega}(\tilde{\mathbf{x}})
+\frac{m_0}{|\Omega|},
\label{eq:massproj}
\end{equation}
where $m_0$ represents the total mass. For velocity updates, we apply the corresponding zero-mean correction so that the ODE updates do not change the total mass.
In practice, we apply $\Pi_{\mathrm{mass}}$ when sampling the physics-embedded prior, and apply the mean-correction to the predicted velocity field whenever this constraint is required.

By encoding these operators into both the physics-embedded prior and the constraint-preserving velocity projection, PerFlow {starts from a physics-feasible state and} performs rectified-flow sampling directly on the physics-consistent manifold without sampling-time physics guidance. Furthermore, we provide theoretical guarantees for the continuous-time dynamics in Sec.~\ref{sec:theory}, where we prove that the physical constraint satisfaction can be maintained throughout sampling.

\subsection{Training Objective and Conditional Sampling}
\label{sec:train_sample}
We introduce PerFlow's training and inference procedures with sparse observations (Fig.~\ref{fig:model}\textbf{b}). 
{The complete algorithmic procedure is provided in Appendix~\ref{sec:app_algorithm}.}
\paragraph{Training (Conditional RF with physics embedding).}
During training, we take a physical field $\mathbf{x}_1$ from the training set and sample a masking $\mathcal{M}$ to simulate sparse measurements $\mathbf{y}=\mathcal{M}\odot \mathbf{x}_1$, thus forming the condition $\mathbf{c}=\mathrm{concat}(\mathbf{y},\mathcal{M})$.
After that, we sample a random noise $\tilde{\mathbf{x}}_0$ from an initial distribution $p_0$ and project it to a physics-consistent initial state $\mathbf{x}_0=\Pi(\tilde{\mathbf{x}}_0)$ using the projection operators in Sec.~\ref{sec:physics_manifold}. We then construct $\mathbf{x}_t$ and $\mathbf{v}^\star$ as in Eqs.~\ref{eq:rf_path}-\ref{eq:rf_target}, and minimize
\begin{equation}
\mathcal{L}(\theta)
=
\mathbb{E}\Big[\big\|
\mathbf{v}_\theta(t,\mathbf{x}_t,\mathbf{c})-\mathbf{v}^\star
\big\|_2^2\Big].
\label{eq:training_loss}
\end{equation}

In our implementation, $\mathbf{v}_\theta$ is made physics-consistent by constraint-preserving projection (Sec.~\ref{sec:physics_manifold}), thus the learned dynamics is restricted to the physics-consistent manifold automatically.

\paragraph{Inference (Sparse reconstruction and uncertainty quantification).}

Given sparse observations $(\mathbf{y},\mathcal{M})$, we set $\mathbf{c}=\mathrm{concat}(\mathbf{y},\mathcal{M})$ to form a condition. We sample $K$ i.i.d. physics-consistent initial states and integrate the conditional ODE in Eq.~\ref{eq:rf_ode} from $t=0$ to $1$ using a few-step solver to obtain samples $\hat{\mathbf{x}}^{(k)}=\mathbf{x}^{(k)}(1)$. After that, we can estimate the uncertainty from the ensemble $\{\hat{\mathbf{x}}^{(k)}\}_{k=1}^K$ via posterior mean and variance.

\section{Theory} \label{sec:theory} 
This section states theoretical guarantees for PerFlow's hard-constraint embedding. After discretization, many physics priors (e.g., conservation and boundary conditions) can be written as affine linear constraints \begin{equation} 
\mathcal{S} \;=\; \{\mathbf{x}\in\mathbb{R}^d:\; A\mathbf{x}= \mathbf{p}\}, \label{eq:theory_S} \end{equation} 
where $A$ stacks the corresponding constraint operators and $\mathbf{p}$ encodes prescribed values. 
PerFlow generates samples by integrating the conditional ODE 
\begin{equation} 
\frac{d\mathbf{x}}{dt}=\mathbf{v}_\theta(t,\mathbf{x},\mathbf{c}), \label{eq:theory_ode} 
\end{equation} 
where $\mathbf{c}$ is the observation condition. 
{By design, PerFlow obtains a feasible initial state $\mathbf{x}(0)\in\mathcal{S}$ via prior projection and uses a constraint-preserving velocity field (Sec.~\ref{sec:physics_manifold}). The following theorem formalizes the resulting trajectory-level constraint preservation.}
\begin{theorem}[Constraint invariance of continuous dynamics] \label{thm:continuous_invariance_short} 
Assume the feasible set is $\mathcal{S}=\{\mathbf{x}:A\mathbf{x}=\mathbf{p}\}$. Assume $\mathbf{v}_\theta(t,\mathbf{x},\mathbf{c})$ is continuous in $t$ and locally Lipschitz in $\mathbf{x}$ so that the ODE~\eqref{eq:theory_ode} admits a unique solution on $[0,1]$. 
If (i) $\mathbf{x}(0)\in\mathcal{S}$ and (ii) the velocity field is tangent to $\mathcal{S}$, i.e., \begin{equation} 
A\,\mathbf{v}_\theta(t,\mathbf{x},\mathbf{c})=\mathbf{0} \quad \text{for all } t\in[0,1]\ \text{and all }\mathbf{x}\in\mathcal{S}, \label{eq:tangent_condition_short} 
\end{equation} 
then the solution satisfies $\mathbf{x}(t)\in\mathcal{S}$ for all $t\in[0,1]$. 
\end{theorem} 
In PerFlow, we parameterize the velocity field in a constrained form (e.g., stream-function for incompressibility) or use an exact projection to the network output to enforce the corresponding constraints (mean correction for mass conservation).
As a result, for any feasible state $\mathbf{x}\in\mathcal{S}$, the effective velocity satisfies $A\,\mathbf{v}_\theta(t,\mathbf{x},\mathbf{c})=\mathbf{0}$, which is exactly the tangency condition in Eq.~\ref{eq:tangent_condition_short}. 
{Therefore, with prior-projected initialization and constraint-preserving velocity projection, physical constraint satisfaction can be maintained throughout the entire sampling (detailed in Appendix~\ref{sec:app_theory}).}
{This trajectory-level guarantee differs from sampling-time guidance, which encourages physical consistency through additional correction terms during sampling but does not enforce feasibility for every ODE update. 
For nonlinear manifolds $\mathcal{M}=\{\mathbf{x}:g(\mathbf{x})=0\}$, trajectory invariance would require $J_g(\mathbf{x})\mathbf{v}_\theta(t,\mathbf{x},\mathbf{c})=0$, where $J_g(\mathbf{x})$ is the Jacobian of $g$ at $\mathbf{x}$; we leave this extension for future work.}

\begin{table*}[!t]
\centering
\resizebox{\textwidth}{!}{
\begin{tabular}{ll|cccccc|cc}
\toprule
PDE & Metric
& UNet & FNO & FFNO & DiffPDE-50 & S$^3$GM-50 & PerFlow & DiffPDE-2000 & S$^3$GM-2000 \\
\midrule

\multirow{4}{*}{Burgers}
& Rel-$\ell_2$ & $1.86 \times 10^{-1}$ & $\underline{1.23 \times 10^{-1}}$ & $1.93 \times 10^{-1}$ & $1.45 \times 10^{0}$ & $6.99 \times 10^{0}$ & $\mathbf{8.63 \times 10^{-2}}$ & $9.39 \times 10^{-2}$ & $1.87 \times 10^{-1}$ \\
& Rel-$\ell_1$ & $1.33 \times 10^{-1}$ & $\underline{8.09 \times 10^{-2}}$ & $1.49 \times 10^{-1}$ & $1.47 \times 10^{0}$ & $6.81 \times 10^{0}$ & $\mathbf{6.84 \times 10^{-2}}$ & $5.92 \times 10^{-2}$ & $1.60 \times 10^{-1}$ \\
& Phys-Err      & $2.03 \times 10^{-4}$ & $\underline{4.62 \times 10^{-5}}$ & $4.23 \times 10^{-3}$ & $1.19 \times 10^{-4}$ & $7.49 \times 10^{-3}$ & $\mathbf{4.47 \times 10^{-11}}$ & $1.53 \times 10^{-4}$ & $1.04 \times 10^{-4}$ \\
& Time     & 35.88   & 29.02   & 29.97   & 4.47    & \underline{3.33}    & \textbf{0.79}    & 191.90     & 145.60 \\
\midrule

\multirow{4}{*}{Poisson}
& Rel-$\ell_2$ & $2.46 \times 10^{-1}$ & $\underline{1.66 \times 10^{-1}}$ & $3.66 \times 10^{-1}$ & $9.16 \times 10^{-1}$ & $3.90 \times 10^{1}$ & $\mathbf{5.02 \times 10^{-2}}$ & $5.12 \times 10^{-2}$ & $7.70 \times 10^{-2}$ \\
& Rel-$\ell_1$ & $1.97 \times 10^{-1}$ & $\underline{1.34 \times 10^{-1}}$ & $3.23 \times 10^{-1}$ & $9.23 \times 10^{-1}$ & $3.96 \times 10^{1}$ & $\mathbf{4.63 \times 10^{-2}}$ & $4.49 \times 10^{-2}$ & $7.11 \times 10^{-2}$ \\
& Phys-Err      & $1.13 \times 10^{-2}$ & $4.13 \times 10^{-2}$ & $2.23 \times 10^{-2}$ & $\underline{8.37 \times 10^{-3}}$ & $6.88 \times 10^{1}$ & $\mathbf{0.00 \times 10^{0}}$ & $7.21 \times 10^{-2}$ & $6.21 \times 10^{-2}$ \\
& Time     & 35.69   & 29.13   & 30.97   & 5.58    & \underline{3.70}    & \textbf{0.82}    & 261.40     & 147.30 \\
\midrule

\multirow{4}{*}{Darcy}
& Rel-$\ell_2$ & $2.43 \times 10^{-1}$ & $\underline{1.71 \times 10^{-1}}$ & $2.37 \times 10^{-1}$ & $3.95 \times 10^{-1}$ & $2.92 \times 10^{0}$ & $\mathbf{6.58 \times 10^{-2}}$ & $1.15 \times 10^{-1}$ & $2.09 \times 10^{-1}$ \\
& Rel-$\ell_1$ & $1.27 \times 10^{-1}$ & $\underline{6.66 \times 10^{-2}}$ & $1.23 \times 10^{-1}$ & $2.99 \times 10^{-1}$ & $2.72 \times 10^{0}$ & $\mathbf{1.13 \times 10^{-2}}$ & $3.33 \times 10^{-2}$ & $3.87 \times 10^{-2}$ \\
& Phys-Err      & $3.39 \times 10^{-2}$ & $\underline{1.15 \times 10^{-2}}$ & $8.08 \times 10^{-2}$ & $1.24 \times 10^{-1}$ & $1.65 \times 10^{1}$ & $\mathbf{1.05 \times 10^{-7}}$ & $9.45 \times 10^{-2}$ & $3.95 \times 10^{-1}$ \\
& Time    & 36.55   & 29.51   & 31.49   & 5.85    & \underline{3.98}    & \textbf{0.81}    & 257.90     & 154.20 \\
\midrule

\multirow{4}{*}{NS}
& Rel-$\ell_2$ & $4.30 \times 10^{-1}$ & $4.64 \times 10^{-1}$ & $\underline{3.68 \times 10^{-1}}$ & $1.14 \times 10^{0}$ & $6.66 \times 10^{1}$ & $\mathbf{2.22 \times 10^{-2}}$ & $3.17 \times 10^{-3}$ & $1.26 \times 10^{-1}$ \\
& Rel-$\ell_1$ & $2.57 \times 10^{-1}$ & $4.03 \times 10^{-1}$ & $\underline{2.52 \times 10^{-1}}$ & $1.15 \times 10^{0}$ & $6.33 \times 10^{1}$ & $\mathbf{1.76 \times 10^{-2}}$ & $2.15 \times 10^{-3}$ & $1.20 \times 10^{-1}$ \\
& Phys-Err      & $9.83 \times 10^{0}$ & $1.06 \times 10^{1}$ & $1.39 \times 10^{1}$ & $\underline{2.62 \times 10^{-3}}$ & $6.09 \times 10^{5}$ & $\mathbf{8.04 \times 10^{-10}}$ & $2.50 \times 10^{-4}$ & $2.16 \times 10^{0}$ \\
& Time     & 37.55   & 35.56   & 55.07   & 7.67    & \underline{3.61}    & \textbf{1.20}    & 292.70     & 155.60 \\
\bottomrule
\end{tabular}
}
\caption{Quantitative comparison on different PDEs, including relative $\ell_2$ error (Rel-$\ell_2$), relative $\ell_1$ error (Rel-$\ell_1$), physics error (Phys-Err), and inference time (s). We evaluate S$^3$GM and DiffPDE with a few-step sampling (50 steps) and a long-run setting (2000 steps), denoted as S$^3$GM-50 (2000) and DiffPDE-50 (2000). PerFlow is evaluated with 50 sampling steps. Best results are in bold and second best are underlined (2000-step baselines are excluded from ranking).}
\label{tab:main_table}
\end{table*}

\begin{figure*}[t!]
    \centering
    \includegraphics[width=\textwidth]{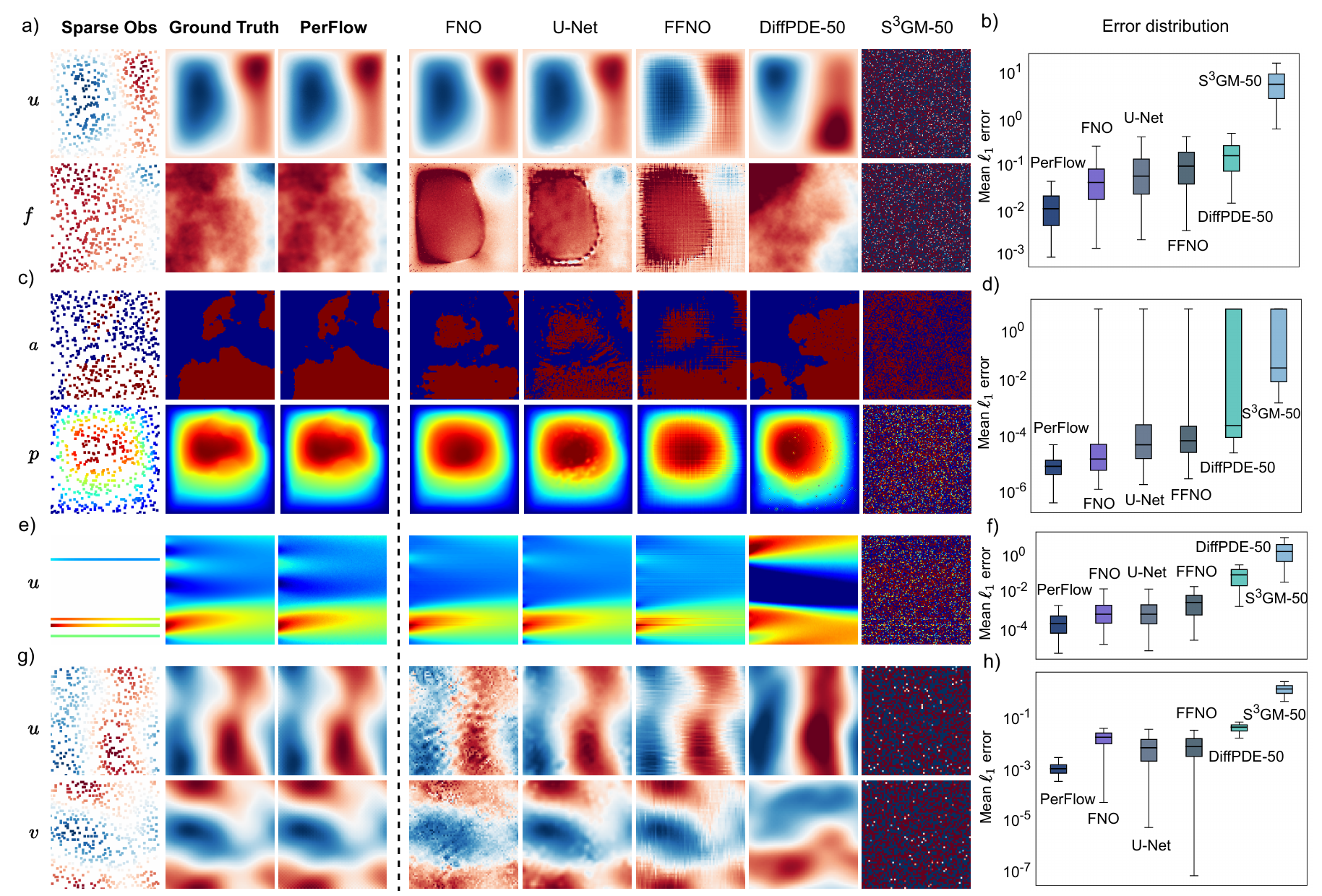}
    \caption{Sparse reconstruction physical fields (left) and the distribution of the point-wise mean $\ell_1$ error (right) on four PDE benchmarks for each methods: \textbf{(a-b)} Poisson (solution and forcing), \textbf{(c-d)} Darcy (coefficient and pressure), \textbf{(e-f)} Burgers, and \textbf{(g-h)} Navier--Stokes equation (velocity components $u$ and $v$).}
    \label{fig:main_plot}
\end{figure*}

\section{Results}\label{sec:results}

\subsection{Benchmarks, Baselines, and Metrics}\label{sec:results_setup}

\paragraph{Benchmarks.}

We evaluate PerFlow on four kinds of PDE systems, including: (i) 1D \emph{Burgers} with periodic boundary conditions, (ii) 2D \emph{Poisson} with Dirichlet boundary conditions, (iii) 2D \emph{Darcy} with Dirichlet boundary conditions, and (iv) 2D \emph{incompressible Navier--Stokes} (NS) on a periodic domain. The default discretizations and measurements are as follows: Burgers uses $128$ time steps on a $128$-point spatial grid with $k=5$ observed locations; Poisson and Darcy use $128 \times 128$ grids with $k=500$ observed points; NS uses 2D velocity fields $(u,v)$ on a $64 \times 64$ grid over $T=10$ frames, with $k=200$ observed locations per frame. Observation masks are randomly sampled. Additional dataset details are provided in Appendix~\ref{sec:dataset_info} and Appendix~\ref{sec:data_usage}.

\paragraph{Baselines.}

We compare PerFlow with (i) deterministic surrogates and (ii) guided generative baselines. For deterministic models, we train \textbf{FNO}~\cite{li2020fno}, \textbf{U-Net}~\cite{Ronneberger2015unet}, and \textbf{FFNO}~\cite{tran2023factorized} as \emph{surrogate solvers}, learning the forward mapping between \emph{full fields} governed by the PDE. For sparse reconstruction, we randomly initialize the full fields as input variables and iteratively optimize them by minimizing the mismatch between the model's output and the sparse observations. The reported inference time includes the entire optimization process. 
Furthermore, we also compare PerFlow with conditional FNO, which learns a deterministic mapping from sparse observations to the full field directly, and the corresponding results are reported in Appendix~\ref{sec:fno_condition}.

For generative baselines, we evaluate S$^3$GM~\cite{li2024s3gm} and DiffPDE~\cite{huang2024diffusionpde}, each with a few-step sampling (50 steps) and a long-run setting following the sampling steps used in the original papers (2000 steps), denoted as S$^3$GM-50 (2000) and DiffPDE-50 (2000), respectively.
{These two settings highlight the accuracy--efficiency trade-off: 2000 steps reflect the high-budget performance of guided baselines, while 50 steps compare them with PerFlow under a similar inference budget.
Implementation details and baseline descriptions are provided in Appendix~\ref{sec:training_details} and Appendix~\ref{sec:baseline_models}.}
{\textbf{The code will be released at \url{https://github.com/intell-sci-comput/PerFlow}.}}

\paragraph{Metrics.} In this paper, we report reconstructed accuracy via relative errors, including $\text{Rel-}\ell_2=\|\hat{\mathbf{x}}-\mathbf{x}\|_2/\|\mathbf{x}\|_2$ and 
$\text{Rel-}\ell_1=\|\hat{\mathbf{x}}-\mathbf{x}\|_1/\|\mathbf{x}\|_1$.
We also report the physics error (Phys-Err) for each PDE, such as divergence for NS and mass conservation for Burgers.
Full definitions, including component averaging across variables and the PDE-specific Phys-Err computation, are provided in Appendix~\ref{sec:eval_metrics}.

\subsection{Sparse Reconstruction of Physical Fields}\label{sec:main_results}

Fig.~\ref{fig:main_plot} shows the sparse reconstruction results for each PDE. Among these methods, deterministic models often offer over-smoothed reconstructed results and introduce artifacts in unobserved regions. Meanwhile, guided diffusion baselines, like DiffPDE and S$^3$GM, can recover several detail patterns and improve observation consistency when sufficient sampling steps are taken, but still struggle in the few-step regime.

Table~\ref{tab:main_table} summarizes the quantitative results across 20 test cases for each PDE. PerFlow significantly outperforms baselines in reconstruction accuracy while strictly enforcing physical constraints. Compared to deterministic surrogates, PerFlow reduces Rel-$\ell_2$ errors by \emph{29.8\%--94.0\%} across the four benchmarks. This is because solving the inverse problem via iterative optimization with a deterministic model is ill-posed and sensitive to initialization, resulting in higher errors and prolonged inference times (29--55s). Moreover, compared with conditional FNO, PerFlow still achieves a remarkably high accuracy due to its ability to handle ill-posed scenarios (Appendix~\ref{sec:fno_condition}). 

When compared with generative models under the same 50-step budget, PerFlow improves reconstruction accuracy by \emph{1--2 orders of magnitude} over DiffPDE-50 and S$^3$GM-50. Simultaneously, PerFlow is \emph{5.7--6.4$\times$} faster than DiffPDE-50 and \emph{3.0--4.2$\times$} faster than S$^3$GM-50.
This efficiency stems from PerFlow's design. Guided baselines require per-step backpropagation for observation and physics guidance, while PerFlow performs guidance-free sampling with lightweight projections. 
{For example, on Burgers, mass projection introduces negligible overhead, with $0.79\pm0.12$s using projection versus $0.76\pm0.06$s without it.}
Even when the guided baselines are extended to 2000 sampling steps, PerFlow maintains superior accuracy on three of the four benchmarks while achieving significantly strong physics consistency. Notably, PerFlow provides up to a $\sim$320$\times$ speedup over DiffPDE-2000. In summary, PerFlow can not only offer high fidelity comparable to converged diffusion models, but also enjoy fast inference speed due to few-step solvers. 
Additional robustness analyses on base distributions and noisy observations are provided in the Appendix~\ref{sec:app_robustness}.
Moreover, long-run guided baselines and varying-observation results are reported in Appendix~\ref{sec:supp_main_results} and Appendix~\ref{sec:more_obs_results}.

\begin{figure*}[!t]
    \centering
    \includegraphics[width=\textwidth]{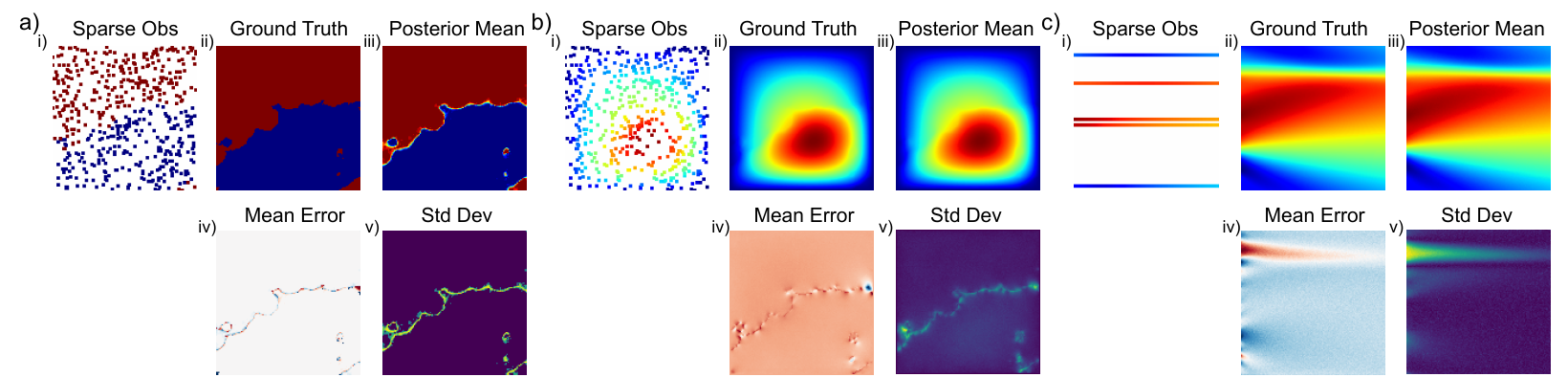}
    \caption{Uncertainty quantification under sparse observations, including Darcy coefficient $a(x)$ (\textbf{a}), Darcy pressure $p$ (\textbf{b}), and 1D Burgers $u$ (\textbf{c}). In each subfigure, panels (i)--(v) correspond to the sparse observations, the ground truth, the predicted mean value, the mean error, and the standard deviation over 20 samples.}
    \label{fig:uq}
\end{figure*}

\subsection{Uncertainty Quantification}\label{sec:uq}
A main advantage of generative modeling is its ability to quantify posterior uncertainty under sparse observations. To verify this, we evaluate this capability on the Darcy and Burgers benchmarks, where the inverse problem is inherently ill-posed due to sparse and irregular measurements. Armed with the generative nature of PerFlow, we generate an ensemble of 20 independent predictions for each test case by sampling distinct initial noise states. We then calculate the predicted mean as the final reconstructed field and the entry-wise standard deviation across the ensemble to quantify the model's uncertainty. Fig.~\ref{fig:uq} shows the uncertainty quantification results. In all cases, the predictive mean accurately recovers the ground truth structures. Furthermore, we observe a strong spatial correlation between the estimated uncertainty and the actual reconstruction error. Regions exhibiting high variance typically coincide with larger deviations from the ground truth. For instance, uncertainty peaks at the sharp phase boundaries in the Darcy coefficient (Fig.~\ref{fig:uq}\textbf{a}) and concentrates in the gaps between temporal snapshots for the Burgers equation (Fig.~\ref{fig:uq}\textbf{c}). This alignment indicates that PerFlow effectively identifies challenging regions where the solution is less confident, serving as a reliable indicator of reconstruction errors.

\begin{figure}[!t]
    \centering
    \includegraphics[width=\linewidth]{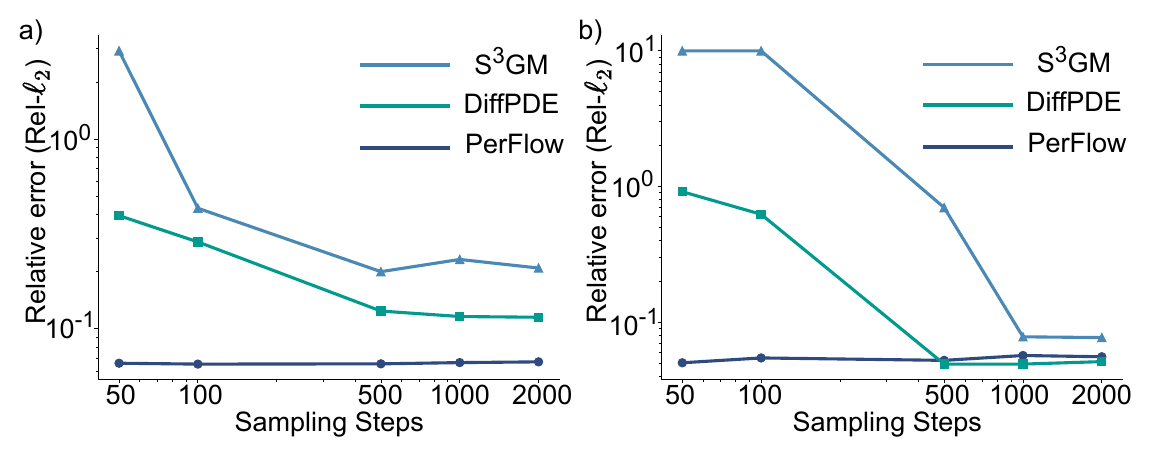}
    \caption{Reconstruction error versus sampling steps on (\textbf{a}) 2D Darcy and (\textbf{b}) 2D Poisson. We report the Rel-$\ell_2$ error versus the number of sampling steps for S$^3$GM, DiffPDE, and PerFlow.}
    \label{fig:darcy_steps}
\end{figure}

\subsection{Few-step Sampling}\label{sec:steps}
We study how reconstruction accuracy varies with the number of sampling steps for S$^3$GM, DiffPDE, and PerFlow (Fig.~\ref{fig:darcy_steps}). Across these PDEs, PerFlow achieves strong accuracy in the few-step sampling regime and remains stable as the number of steps increases. In contrast, other guided diffusion methods exhibit a much slower and more severe trade-off between accuracy and the number of steps, because of the need to balance data-fidelity guidance and physics guidance. PerFlow addresses this issue via avoiding sampling-time backpropagation and embedding physics as hard constraints, thereby achieving efficient and accurate reconstruction.

\begin{table}[!t]
\centering
\resizebox{\linewidth}{!}{
\begin{tabular}{l|ccc}
\toprule
Method & $\mathrm{Rel}\text{-}\ell_2$ & $\mathrm{Rel}\text{-}\ell_1$ & Phys-Err \\
\midrule
PerFlow            & $\mathbf{8.63\times10^{-2}}$ & $\underline{6.84\times10^{-2}}$ & $\mathbf{4.47\times10^{-11}}$ \\
RF-50              & $\underline{8.79\times10^{-2}}$ & $7.16\times10^{-2}$ & $\underline{7.09\times10^{-5}}$  \\
FM-50              & $1.48\times10^{-1}$ & $1.40\times10^{-1}$ & $1.71\times10^{-4}$  \\
DDIM-50            & $1.14\times10^{-1}$ & $9.92\times10^{-2}$ & $1.02\times10^{-4}$  \\
DDPM-50            & $6.86\times10^{0}$  & $7.33\times10^{0}$  & $1.25\times10^{-4}$  \\
DDPM-1000          & $9.10\times10^{-2}$ & $\mathbf{6.46\times10^{-2}}$ & $1.10\times10^{-4}$  \\
\bottomrule
\end{tabular}
}
\caption{Ablation studies across different generative samplers and flow variants on Burgers equation, including relative $\ell_2$ error (Rel-$\ell_2$), relative $\ell_1$ error (Rel-$\ell_1$), physics error (Phys-Err). Best results are in bold and second best are underlined.}
\label{tab:ablation}
\end{table}

\subsection{Ablation Studies}\label{sec:ablation}
We conduct ablations on 1D Burgers equation to study (i) the effect of \emph{physics embedding} and (ii) the choice of \emph{generative procedure}. Table~\ref{tab:ablation} compares PerFlow with several variants, including \emph{RF-50}, which removes the hard constraint projector from PerFlow; \emph{FM-50}, which replaces rectified flow with a standard flow-matching model under the same 50-step sampling~\cite{lipman2022flow}; \emph{DDIM-50}~\cite{song2020ddim} and \emph{DDPM-50}~\cite{ho2020denoising}, two diffusion baselines sampled with 50 steps, and \emph{DDPM-1000} with 1000 steps.

Comparing PerFlow with RF-50 reveals the contribution of hard constraint embedding. While both methods achieve similar reconstruction accuracy, PerFlow reduces physics error by several orders of magnitude (Table~\ref{tab:ablation}), indicating that constraint-preserving projection is crucial for producing physically consistent reconstructions. Furthermore, RF-based generation in PerFlow is more effective than other diffusion models with the same number of sampling (Table~\ref{tab:ablation}). DDPM-50 fails catastrophically, whereas DDIM-50 improves substantially but remains inferior to PerFlow in both accuracy and physical consistency. Although DDPM-1000 can achieve competitive Rel-$\ell_2$/$\ell_1$, it still exhibits larger physics error and requires an order-of-magnitude more sampling steps.

\section{Conclusion and Future Work}
We propose PerFlow, a physics-embedded rectified-flow framework for sparse reconstruction and uncertainty quantification in PDE-governed systems. Compared with posterior-guided generative methods, PerFlow decouples observation conditioning from physics enforcement by performing guidance-free observation conditioning in rectified-flow dynamics and encoding hard physics via constraint-preserving projection. Theoretically, sampling trajectories are guaranteed to remain on the physics-consistent manifold. Numerically, PerFlow achieves competitive accuracy and strong physics consistency across various PDE datasets, while enabling efficient conditional sampling (e.g., 50 steps) and up to $\sim$320$\times$ faster inference than 2000-step diffusion baselines.

There are two future works that we are interested in. Firstly, we can combine PerFlow with sensor placement approaches, which can maximize reconstruction quality and uncertainty reduction~\cite{ma2025physense}. Secondly, we can extend PerFlow to other downstream tasks, because fast conditional sampling on a physics-consistent manifold can provide reliable control and planning~\cite{hu2025from}.

\section*{Acknowledgements}
The work is supported by the Beijing Natural Science Foundation (No. F261002) and the National Natural Science Foundation of China (No. 62276269 and No. 62506367). R.Z. acknowledges support from the China Postdoctoral Science Foundation under Grant Number 2025M771582 and the Postdoctoral Fellowship Program of CPSF under Grant Number GZB20250408. The full version is available at \url{http://arxiv.org/abs/2605.03548}.

\bibliographystyle{named}
\bibliography{ijcai26}

\clearpage 
\appendix
\setcounter{section}{0}
\renewcommand{\thesection}{\Alph{section}}

\onecolumn
\vspace*{0.5em}
\begin{center}
{\LARGE\bfseries Appendix: PerFlow: Physics-Embedded Rectified Flow for Efficient Reconstruction and Uncertainty Quantification of Spatiotemporal Dynamics}
\end{center}
\vspace{20pt}

This appendix provides additional details, including proofs of the main theorems, training settings, baseline descriptions, dataset generation, data-usage studies, evaluation metrics, and supplementary results.

\section{Theoretical Results and Proofs}\label{sec:app_theory}

This appendix provides full statements and proofs for the constraint invariance results underlying PerFlow's hard-constraint embedding.
After discretization, many physics priors used in this work (e.g., incompressibility, Dirichlet boundary conditions, and global conservation) admit an affine linear form.
Let $\mathbf{x}\in\mathbb{R}^d$ denote the vectorized field (or a vectorized spatiotemporal block), and define the feasible set
\begin{equation}
\mathcal{S} \;=\; \{\mathbf{x}\in\mathbb{R}^d:\; A\mathbf{x}=\mathbf{p}\},
\label{eq:app_S}
\end{equation}
where $A$ stacks linear operators corresponding to the hard constraints and $\mathbf{p}$ encodes prescribed values.
PerFlow generates samples by integrating the conditional ODE
\begin{equation}
\frac{d\mathbf{x}}{dt}=\mathbf{v}_\theta(t,\mathbf{x},\mathbf{c}),\qquad t\in[0,1],
\label{eq:app_ode}
\end{equation}
where $\mathbf{c}$ is the observation condition.
A key design in PerFlow is to ensure (i) $\mathbf{x}(0)\in\mathcal{S}$ via the physics-embedded prior and (ii) the velocity field is constraint-preserving, formalized by the tangency condition
\begin{equation}
A\,\mathbf{v}_\theta(t,\mathbf{x},\mathbf{c})=\mathbf{0}
\quad \text{for all } t\in[0,1]\ \text{and all }\mathbf{x}\in\mathcal{S}.
\label{eq:app_tangent}
\end{equation}

\subsection{Constraint invariance in continuous time}

\begin{theorem}[Constraint invariance of continuous dynamics]
    \label{thm:app_continuous_invariance}
    Assume the feasible set is $\mathcal{S}=\{\mathbf{x}:A\mathbf{x}=\mathbf{p}\}$.
    Assume $\mathbf{v}_\theta(t,\mathbf{x},\mathbf{c})$ is continuous in $t$ and locally Lipschitz in $\mathbf{x}$ so that the ODE~\eqref{eq:app_ode} admits a unique solution on $[0,1]$.
    If (i) $\mathbf{x}(0)\in\mathcal{S}$ and (ii) the tangency condition~\eqref{eq:app_tangent} holds, then the ODE solution satisfies $\mathbf{x}(t)\in\mathcal{S}$ for all $t\in[0,1]$.
\end{theorem}

\begin{proof}
    Let $\mathbf{x}(t)$ be the unique solution of~\eqref{eq:app_ode} on $[0,1]$.
    Define the constraint residual
    \[
    \mathbf{r}(t) \;=\; A\mathbf{x}(t)-\mathbf{p}.
    \]
    Since $\mathbf{x}(t)$ is differentiable, $\mathbf{r}(t)$ is differentiable and
    \begin{equation}
    \dot{\mathbf{r}}(t)
    = A\dot{\mathbf{x}}(t)
    = A\,\mathbf{v}_\theta(t,\mathbf{x}(t),\mathbf{c}).
    \label{eq:app_residual_derivative}
    \end{equation}
    The assumption $\mathbf{x}(0)\in\mathcal{S}$ implies $\mathbf{r}(0)=\mathbf{0}$.
    
    To prove $\mathbf{r}(t)\equiv\mathbf{0}$ on $[0,1]$, define the maximal time up to which the trajectory stays feasible:
    \[
    T^\star \,=\, \sup\{\,T\in[0,1]:\; \mathbf{x}(t)\in\mathcal{S}\ \text{for all }t\in[0,T] \,\}.
    \]
    Clearly $T^\star\ge 0$ since $\mathbf{x}(0)\in\mathcal{S}$.
    By continuity of $\mathbf{x}(t)$ and closedness of $\mathcal{S}$, we also have $\mathbf{x}(T^\star)\in\mathcal{S}$.
    
    We show that $T^\star=1$.
    Assume for contradiction that $T^\star<1$.
    Because $\mathcal{S}$ is an affine subspace, pick any particular point $\mathbf{x}_\mathrm{p}\in\mathcal{S}$ and let $N\in\mathbb{R}^{d\times q}$ be a full-column-rank matrix whose columns form a basis of $\ker(A)$, so that $AN=0$.
    Then every feasible state admits the parameterization
    \[
    \mathbf{x}\in\mathcal{S}\quad\Longleftrightarrow\quad \mathbf{x}=\mathbf{x}_\mathrm{p}+N\mathbf{z}\ \text{for some }\mathbf{z}\in\mathbb{R}^q.
    \]
    At time $T^\star$, since $\mathbf{x}(T^\star)\in\mathcal{S}$, there exists $\mathbf{z}^\star$ such that $\mathbf{x}(T^\star)=\mathbf{x}_\mathrm{p}+N\mathbf{z}^\star$.
    
    Now consider the vector field restricted to $\mathcal{S}$.
    For any $\mathbf{x}\in\mathcal{S}$, the tangency condition~\eqref{eq:app_tangent} implies
    $A\mathbf{v}_\theta(t,\mathbf{x},\mathbf{c})=0$, hence $\mathbf{v}_\theta(t,\mathbf{x},\mathbf{c})\in\ker(A)=\mathrm{range}(N)$.
    Fix any left-inverse $N^\dagger\in\mathbb{R}^{q\times d}$ satisfying $N^\dagger N=I_q$ (e.g., the Moore--Penrose pseudoinverse).
    Define the reduced vector field
    \[
    \mathbf{g}(t,\mathbf{z}) \,=\, N^\dagger\,\mathbf{v}_\theta\big(t,\mathbf{x}_\mathrm{p}+N\mathbf{z},\mathbf{c}\big).
    \]
    Because $\mathbf{v}_\theta$ is continuous in $t$ and locally Lipschitz in $\mathbf{x}$, the induced map $\mathbf{g}(t,\mathbf{z})$ is continuous in $t$ and locally Lipschitz in $\mathbf{z}$.
    Hence the reduced ODE
    \begin{equation}
    \frac{d\mathbf{z}}{dt}=\mathbf{g}(t,\mathbf{z}),\qquad \mathbf{z}(T^\star)=\mathbf{z}^\star,
    \label{eq:app_reduced_ode}
    \end{equation}
    admits a unique solution on $[T^\star,T^\star+\delta]$ for some $\delta>0$.
    Define $\tilde{\mathbf{x}}(t)=\mathbf{x}_\mathrm{p}+N\mathbf{z}(t)$ on this interval.
    Then $\tilde{\mathbf{x}}(t)\in\mathcal{S}$ for all $t\in[T^\star,T^\star+\delta]$, and
    \[
    \frac{d\tilde{\mathbf{x}}}{dt}
    = N\frac{d\mathbf{z}}{dt}
    = N\mathbf{g}(t,\mathbf{z}(t))
    = \mathbf{v}_\theta\big(t,\tilde{\mathbf{x}}(t),\mathbf{c}\big),
    \]
    where the last equality holds because $\mathbf{v}_\theta(t,\tilde{\mathbf{x}}(t),\mathbf{c})\in\mathrm{range}(N)$ and $NN^\dagger$ acts as the identity on $\mathrm{range}(N)$.
    Thus $\tilde{\mathbf{x}}(t)$ is a solution of the original ODE~\eqref{eq:app_ode} with the same initial condition $\tilde{\mathbf{x}}(T^\star)=\mathbf{x}(T^\star)$.
    By uniqueness, $\tilde{\mathbf{x}}(t)=\mathbf{x}(t)$ on $[T^\star,T^\star+\delta]$, implying $\mathbf{x}(t)\in\mathcal{S}$ beyond $T^\star$, which contradicts the definition of $T^\star$.
    Therefore $T^\star=1$ and $\mathbf{x}(t)\in\mathcal{S}$ for all $t\in[0,1]$.
\end{proof}

\section{Algorithmic Procedure}
\label{sec:app_algorithm}

\begin{algorithm}[H]
\caption{PerFlow Training and Inference}
\label{alg:app_perflow}
\small
\DontPrintSemicolon
\KwIn{Training fields $\mathbf{x}_1$; mask sampler $p(\mathcal{M})$; base distribution $p_0$; projection operator $\Pi(\cdot)$; number of samples $K$; ODE steps $N$.}
\KwOut{Reconstruction samples $\{\hat{\mathbf{x}}^{(k)}\}_{k=1}^{K}$.}

\BlankLine
\textbf{Training (conditional rectified-flow matching).}\;
\While{not converged}{
  Sample a minibatch of ground-truth fields $\mathbf{x}_1$\;
  Sample masks $\mathcal{M}\sim p(\mathcal{M})$ and set $\mathbf{y}\leftarrow \mathcal{M}\odot \mathbf{x}_1$\;
  Form the condition $\mathbf{c}\leftarrow \mathrm{concat}(\mathbf{y},\mathcal{M})$\;
  Sample $t\sim\mathcal{U}(0,1)$ and $\tilde{\mathbf{x}}_0\sim p_0$\;
  Apply prior projection: $\mathbf{x}_0\leftarrow \Pi(\tilde{\mathbf{x}}_0)$\;
  Compute $\mathbf{x}_t\leftarrow (1-t)\mathbf{x}_0+t\mathbf{x}_1$ and $\mathbf{v}^{\star}\leftarrow \mathbf{x}_1-\mathbf{x}_0$\;
  Predict $\hat{\mathbf{v}}\leftarrow \mathbf{v}_\theta(t,\mathbf{x}_t,\mathbf{c})$\;
  Update $\theta$ by minimizing $\|\hat{\mathbf{v}}-\mathbf{v}^{\star}\|_2^2$\;
}

\BlankLine
\textbf{Inference (few-step conditional ODE sampling).}\;
Given $(\mathbf{y},\mathcal{M})$, set $\mathbf{c}\leftarrow \mathrm{concat}(\mathbf{y},\mathcal{M})$\;
\For{$k \leftarrow 1$ \KwTo $K$}{
  Sample $\tilde{\mathbf{x}}_0^{(k)}\sim p_0$ and apply prior projection:
  $\mathbf{x}^{(k)}(0)\leftarrow \Pi(\tilde{\mathbf{x}}_0^{(k)})$\;
  Integrate $\frac{d\mathbf{x}^{(k)}}{dt}
  =\mathbf{v}_\theta(t,\mathbf{x}^{(k)},\mathbf{c})$ from $t=0$ to $1$ with $N$ steps\;
  Set $\hat{\mathbf{x}}^{(k)}\leftarrow \mathbf{x}^{(k)}(1)$\;
}
Return $\{\hat{\mathbf{x}}^{(k)}\}_{k=1}^{K}$ and ensemble statistics.\;
\end{algorithm}
\FloatBarrier

\section{Training Details}\label{sec:training_details}
All experiments were conducted on a single NVIDIA A100 (80 GB) GPU with an Intel Xeon Platinum 8380 CPU (2.30GHz, 64 cores).
Unless otherwise specified, we use a unified training recipe across benchmarks and only vary the case-dependent hyperparameters in Table~\ref{tab:train_hparams}.

\begin{table}[H]
    \centering
    \begin{tabular}{l|ccc}
    \toprule
    Case & Batch size & Num.\ epochs & Learning rate \\
    \midrule
    1D Burgers & 24  & 300 & $1\times10^{-4}$ \\
    2D Darcy   & 24  & 500 & $1\times10^{-4}$ \\
    2D Poisson & 24  & 500 & $1\times10^{-4}$ \\
    2D NS      & 10  & 800 & $1\times10^{-4}$ \\
    \bottomrule
    \end{tabular}
    \caption{Training hyperparameters for different cases.}
    \label{tab:train_hparams}
\end{table}

We optimize all models using AdamW with weight decay $1\times10^{-4}$.
We adopt a learning-rate schedule with 10-epoch warmup followed by cosine decay from $1\times10^{-4}$ to $6\times10^{-5}$.
We additionally maintain an exponential moving average (EMA) of parameters with decay rate $0.995$, and report results using the EMA weights for evaluation and sampling.
This shared setup yields stable optimization behavior across different PDE systems.

\section{Baseline Models}\label{sec:baseline_models}

\textbf{Deterministic surrogates (FNO/FFNO/U-Net).}
We use \textbf{FNO}~\cite{li2020fno}, \textbf{FFNO}~\cite{tran2023factorized}, and \textbf{U-Net}~\cite{Ronneberger2015unet} as deterministic \emph{surrogate solvers} that learn the forward operator on each benchmark.
Specifically, the surrogate learns the forward mapping between \emph{full fields}:
for \textbf{Poisson} and \textbf{Darcy}, it maps the PDE input field (e.g., forcing or coefficient/permeability) to the corresponding solution field;
for \textbf{Navier--Stokes}, it maps a past window to a future window (e.g., the first 5 frames $\rightarrow$ the next 5 frames);
for \textbf{Burgers}, it maps the initial condition to the subsequent trajectory.
To perform sparse reconstruction with these surrogates, we adopt test-time optimization. We randomly initialize the unknown full field(s), keep the surrogate parameters fixed, and optimize the input variable by minimizing a data-fidelity objective between the surrogate-predicted output and the sparse observations (optionally with a simple regularizer).
Unless otherwise specified, the reported runtime for these baselines includes the full test-time optimization procedure.

\textbf{Fourier Neural Operator (FNO)}~\cite{li2020fno}.
FNO learns a discretization-invariant operator by parameterizing global convolution kernels in the Fourier domain.
Each Fourier layer applies FFT, multiplies truncated Fourier modes with learned complex weights, and maps back via inverse FFT, combined with pointwise channel mixing.

\textbf{Factorized Fourier Neural Operator (FFNO)}~\cite{tran2023factorized}.
FFNO extends FNO by factorizing the spectral representation and improving residual connections, enabling deeper operator networks with improved efficiency.
In particular, it decomposes multi-dimensional spectral mixing into factorized operations, reducing the computational and memory cost while preserving global receptive fields.

\textbf{U-Net}~\cite{Ronneberger2015unet}.
We use a standard encoder--decoder U-Net with skip connections as a convolutional surrogate for field-to-field regression.
The encoder progressively downsamples the input to extract multi-scale features, while the decoder upsamples to recover spatial resolution.
Skip connections between mirrored encoder and decoder stages preserve fine-scale details and facilitate stable optimization, making U-Net a strong deterministic baseline for PDE field reconstruction.

\textbf{S$^3$GM}~\cite{li2024s3gm}.
S$^3$GM is a \emph{score-based generative model} that learns a generative prior over spatiotemporal fields to capture the underlying dynamics from training data.
For sparse reconstruction, it performs conditional sampling by constructing sparse observations and applying sampling-time guidance based on an observation-consistency (data-fidelity) loss, so that generated samples match the measured entries.
Following the official guided sampling procedure, we evaluate S$^3$GM under a practical few-step budget (50 steps) and a long-run setting using the number of sampling steps in the original paper (2000 steps), denoted as S$^3$GM-50 and S$^3$GM-2000, respectively.

\textbf{DiffusionPDE (DiffPDE)}~\cite{huang2024diffusionpde}.
DiffusionPDE is a diffusion-based framework for PDE-constrained generation and inverse problems that incorporates observation consistency and PDE constraints through gradient-based guidance during sampling.
At each denoising step, it computes gradients of a data-fidelity term (matching sparse observations) and a physics term (physical constraints) to update the sample.
We report results under two sampling budgets: 50 steps and 2000 steps (as used in the original paper), denoted as DiffPDE-50 and DiffPDE-2000, respectively.

\section{Dataset Information}\label{sec:dataset_info}

\paragraph{Static PDEs: Darcy and Poisson.}
We follow the data-generation procedures for static PDE benchmarks in~\cite{li2020fno,huang2024diffusionpde}.
We first sample a Gaussian random field (GRF) $\mu$ on $\Omega=(0,1)^2$ with zero mean and covariance operator
$\mu \sim \mathcal{N}\!\left(0,\;(-\Delta + 9I)^{-2}\right)$.
Unless otherwise specified, all fields are discretized on a $128\times128$ grid.

\paragraph{Darcy flow.}
We consider a static Darcy problem with homogeneous Dirichlet boundary conditions:
{\small
\begin{equation}
-\nabla\!\cdot\!\big(a(\mathbf{x})\,\nabla p(\mathbf{x})\big)=q(\mathbf{x}),\quad \mathbf{x}\in\Omega,
\qquad
p(\mathbf{x})=0,\quad \mathbf{x}\in\partial\Omega,
\label{eq:darcy}
\end{equation}
}
where $p$ denotes the pressure (hydraulic head) field and $a(\mathbf{x})$ is the permeability.
Following~\cite{huang2024diffusionpde}, we construct a binary permeability field by thresholding the GRF:
\begin{equation}
a(\mathbf{x}) =
\begin{cases}
12, & \mu(\mathbf{x}) \ge 0,\\
3,  & \mu(\mathbf{x}) < 0,
\end{cases}
\label{eq:darcy_binary_a}
\end{equation}
and set a constant source term $q(\mathbf{x})\equiv 1$.
We solve for $p$ using second-order finite differences on the grid.

\paragraph{Poisson equation.}
We consider the static Poisson equation with homogeneous Dirichlet boundary conditions:
\begin{equation}
\nabla^2 u(\mathbf{x}) = f(\mathbf{x}),\quad \mathbf{x}\in\Omega,
\qquad
u(\mathbf{x})=0,\quad \mathbf{x}\in\partial\Omega.
\label{eq:poisson}
\end{equation}
Following~\cite{huang2024diffusionpde}, we set the source term as the GRF sample $f(\mathbf{x})=\mu(\mathbf{x})$ and solve for $u$ using second-order finite differences.
To enforce the boundary condition exactly, we apply a smooth boundary-enforcing multiplier $\sin(\pi x_1)\sin(\pi x_2)$ (for $\mathbf{x}=(x_1,x_2)\in(0,1)^2$) to the resulting solution field, consistent with~\cite{huang2024diffusionpde}.
Both $f$ and $u$ have resolution $128\times128$.

\paragraph{Burgers' equation.}
We study the one-dimensional viscous Burgers' equation on $\Omega=(0,1)$ with periodic boundary conditions and viscosity $\nu=0.01$:
{\small
\begin{equation}
\frac{\partial u(x,t)}{\partial t}+\frac{\partial}{\partial x}\!\left(\frac{u(x,t)^2}{2}\right)
=\nu\,\frac{\partial^2 u(x,t)}{\partial x^2},
\quad x\in\Omega,\ t\in(0,T].
\label{eq:burgers}
\end{equation}
}
The initial condition is $u(x,0)=u_0(x)$.
Following~\cite{huang2024diffusionpde}, we sample $u_0$ from a Gaussian random field
$u_0 \sim \mathcal{N}\!\big(0,\;625(-\Delta+25I)^{-2}\big)$
and solve the PDE using a spectral method.
We simulate a horizon of $T=1$~s and record 128 snapshots (including $t{=}0$), resulting in a $128\times128$ time--space field $u_{0:T}$.

\paragraph{2D incompressible Navier--Stokes.}
We consider the two-dimensional incompressible Navier--Stokes equations on $\Omega=(0,1)^2$:
\begin{equation}
\frac{\partial \mathbf{u}}{\partial t}+(\mathbf{u}\cdot\nabla)\mathbf{u}
=\frac{1}{\mathrm{Re}}\nabla^2\mathbf{u}-\nabla p+\mathbf{f},
\qquad
\nabla\cdot\mathbf{u}=0,
\label{eq:ns2d}
\end{equation}
where $\mathbf{u}=(u,v)$ is the velocity field and $p$ is the scalar pressure.
We follow the generation procedure in~\cite{li2020fno}, the initial condition is sampled as a GRF in the vorticity/stream-function formulation, and a fixed forcing pattern is applied,
\begin{equation}
\begin{aligned}
q(\mathbf{x})&=\frac{1}{10}\big(\sin(2\pi(x_1+x_2))+\cos(2\pi(x_1+x_2))\big),\\
&\hfill \mathbf{x}=(x_1,x_2)\in\Omega.
\end{aligned}
\label{eq:ns_forcing}
\end{equation}

We set the Reynolds number to $\mathrm{Re}=1000$ and solve the dynamics using a pseudo-spectral method with Fourier transforms.
We simulate a horizon of $T=1$~s and record 10 snapshots over $[0,T]$.
In our experiments, the pseudo-spectral solver runs on a $256\times256$ grid, and we store the velocity fields $(u,v)$ after downsampling to a $64\times64$ grid.

\section{Data usage}\label{sec:data_usage}

\begin{table}[h]
    \centering
    \begin{tabular}{l|ccc|c}
    \toprule
    PDE & Train & Val & Test & Fields \\
    \midrule
    Burgers & 5000 & 1000 & 20 & $u_{0:T}$ \\
    Darcy   & 5000 & 1000 & 20 & $(a,\,p)$ \\
    Poisson & 5000 & 1000 & 20 & $(f,\,u)$ \\
    NS      & 5000 & 1000 & 20 & $(u,\,v)_{0:T}$ \\
    \bottomrule
    \end{tabular}
    \caption{Dataset splits for each benchmark.}
    \label{tab:data_splits}
\end{table}

Table~\ref{tab:data_splits} summarizes the dataset splits used in all experiments.
Unless otherwise specified, we train each model with 5000 training samples, select hyperparameters using 1000 validation samples, and report results on 20 held-out test samples.

\begin{figure}[h]
    \centering
    \includegraphics[width=8cm]{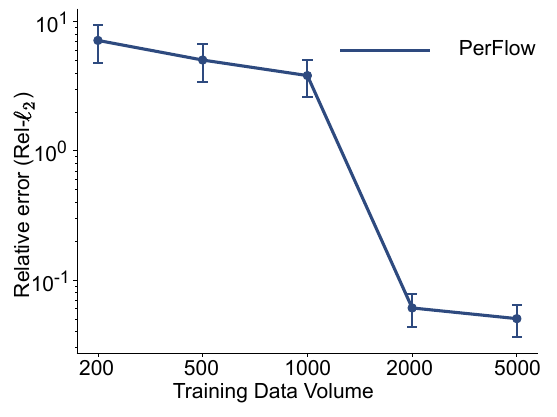}
    \caption{Poisson data-scaling study for PerFlow: test Rel-$\ell_2$ error versus the number of training samples.}
    \label{fig:poisson_data}
\end{figure}

To examine the data efficiency of PerFlow, we conduct an additional study on the Poisson benchmark by varying the number of training samples while keeping the validation/test sets and all training settings fixed.
Fig.~\ref{fig:poisson_data} reports the test Rel-$\ell_2$ error of PerFlow under varying training-set sizes (200/500/1000/2000/5000).
We observe that performance improves monotonically as more training data become available and begins to saturate around 5000 samples, indicating that the default training budget used in our main experiments is sufficient for learning a strong conditional prior.
With substantially fewer samples (e.g., 200--1000), the model exhibits noticeably larger errors, suggesting that limited data restrict the model's ability to capture the conditional distribution under sparse observations.

\section{Evaluation Metrics}\label{sec:eval_metrics}

We report three metrics: relative $\ell_2$ error (Rel-$\ell_2$), relative $\ell_1$ error (Rel-$\ell_1$), and a physics-based constraint error (Phys-Err).
Given ground truth $\mathbf{x}$ and reconstruction $\hat{\mathbf{x}}$, we define
\begin{equation}
\mathrm{Rel}\text{-}\ell_2(\hat{\mathbf{x}},\mathbf{x})
=\frac{\|\hat{\mathbf{x}}-\mathbf{x}\|_2}{\|\mathbf{x}\|_2},
\qquad
\mathrm{Rel}\text{-}\ell_1(\hat{\mathbf{x}},\mathbf{x})
=\frac{\|\hat{\mathbf{x}}-\mathbf{x}\|_1}{\|\mathbf{x}\|_1}.
\label{eq:rel_l2_l1}
\end{equation}
For benchmarks involving multiple target fields, we report component-averaged relative errors:
\begin{equation}
    \mathrm{Rel}\text{-}\ell_p
    =\frac{1}{m}\sum_{j=1}^{m}\mathrm{Rel}\text{-}\ell_p(\hat{\mathbf{x}}^{(j)},\mathbf{x}^{(j)}),
    \quad p\in\{1,2\}.
    \label{eq:avg_rel_lp}
\end{equation}
Concretely, we use: \textbf{Poisson}--average over $(u,f)$; \textbf{Darcy}--average over $(a,p)$; \textbf{Burgers}--single field $u$; \textbf{Navier--Stokes}--average over velocity components $(u,v)$.

Phys-Err measures violations of the \emph{hard physical constraints} used in our physics-embedding, rather than PDE residual losses.
Let $\partial\Omega$ denote the boundary grid and $|\partial\Omega|$ its number of grid points.
For Dirichlet constraints, we compute the mean squared boundary magnitude:
\begin{equation}
\mathrm{Phys\text{-}Err}_{\mathrm{bc}}(\hat{s})
=\frac{1}{|\partial\Omega|}\sum_{\mathbf{x}\in\partial\Omega}\big|\hat{s}(\mathbf{x})\big|^2,
\label{eq:phys_bc}
\end{equation}
where $\hat{s}$ is the predicted field subject to homogeneous Dirichlet boundary conditions.
Thus, for \textbf{Poisson} we apply Eq.~\eqref{eq:phys_bc} to $\hat{u}$, and for \textbf{Darcy} we apply it to $\hat{p}$.

For \textbf{Burgers}, we measure mass-conservation violation by the deviation of the spatial integral from its initial value:
\begin{equation}
\mathrm{Phys\text{-}Err}_{\mathrm{mass}}
=\frac{1}{N_t}\sum_{n=1}^{N_t}\left|\int_{\Omega}\hat{u}(x,t_n)\,dx-\int_{\Omega}u(x,0)\,dx\right|^2,
\label{eq:phys_mass}
\end{equation}
where $\{t_n\}_{n=1}^{N_t}$ are the recorded time steps. In our dataset $\int_{\Omega}u(x,0)\,dx=0$, so Eq.~\eqref{eq:phys_mass} reduces to averaging the squared mass over time.

For \textbf{Navier--Stokes}, we measure incompressibility by the mean squared divergence:
\begin{equation}
\mathrm{Phys\text{-}Err}_{\mathrm{div}}
=\frac{1}{|\mathcal{G}|}\sum_{(\mathbf{x},t)\in\mathcal{G}}
\big|\nabla\cdot\hat{\mathbf{u}}(\mathbf{x},t)\big|^2,
\label{eq:phys_div}
\end{equation}
where $\mathcal{G}$ denotes the spatiotemporal grid used for evaluation.
All spatial derivatives (e.g., $\nabla\cdot$) are computed with the same discretization as in our implementation.

\section{Comparison between PerFlow and Conditional FNO}\label{sec:fno_condition}
\begin{table*}[!t]
\centering
\begin{tabular}{ll|cc}
\toprule
PDE & Metric & Conditional FNO & PerFlow \\
\midrule
\multirow{4}{*}{Burgers}
& Rel-$\ell_2$ & $2.46\times10^{-1}$ & $8.63\times10^{-2}$ \\
& Rel-$\ell_1$ & $2.32\times10^{-1}$ & $6.84\times10^{-2}$ \\
& Phys-Err     & $1.75\times10^{-3}$ & $4.47\times10^{-11}$ \\
& Time (s)     & $0.044$              & $0.79$ \\
\midrule
\multirow{4}{*}{Poisson}
& Rel-$\ell_2$ & $6.65\times10^{-2}$ & $5.02\times10^{-2}$ \\
& Rel-$\ell_1$ & $6.20\times10^{-2}$ & $4.63\times10^{-2}$ \\
& Phys-Err     & $2.65\times10^{-2}$ & $0.00\times10^{0}$ \\
& Time (s)     & $0.062$              & $0.82$ \\
\midrule
\multirow{4}{*}{Darcy}
& Rel-$\ell_2$ & $1.08\times10^{-1}$ & $6.58\times10^{-2}$ \\
& Rel-$\ell_1$ & $3.32\times10^{-2}$ & $1.13\times10^{-2}$ \\
& Phys-Err     & $4.12\times10^{-2}$ & $1.05\times10^{-7}$ \\
& Time (s)     & $0.045$              & $0.81$ \\
\midrule
\multirow{4}{*}{NS}
& Rel-$\ell_2$ & $2.15\times10^{-2}$ & $2.22\times10^{-2}$ \\
& Rel-$\ell_1$ & $1.96\times10^{-2}$ & $1.76\times10^{-2}$ \\
& Phys-Err     & $2.21\times10^{-3}$ & $8.04\times10^{-10}$ \\
& Time (s)     & $0.014$              & $1.20$ \\
\bottomrule
\end{tabular}
\caption{Comparison between Conditional FNO and PerFlow on four PDE benchmarks.}
\label{tab:fno_cond_vs_perflow}
\end{table*}

Table~\ref{tab:fno_cond_vs_perflow} reports an additional comparison between conditional FNO and PerFlow. Conditional FNO directly learns a deterministic mapping from sparse observations to the full field, yielding very fast inference.
However, as a point-estimator, it does not model the posterior distribution under partial observations and thus cannot provide uncertainty quantification.
Moreover, its conditional regression does not explicitly enforce our hard physical constraints, which lead to noticeable constraint violations (Phys-Err), especially in ill-posed regimes.

\section{Additional Robustness Analyses}
\label{sec:app_robustness}

\subsection{Sensitivity to Base Distribution}

We further evaluate the sensitivity of PerFlow to the choice of the base distribution $p_0$.
Table~\ref{tab:base_distribution} compares Gaussian, Uniform, and Gaussian mixture model (GMM) priors on Burgers and Navier--Stokes.
Gaussian and GMM priors achieve similar performance, while Uniform is slightly weaker but remains competitive, suggesting that PerFlow is reasonably robust to the base prior.
\begin{table}[h]
    \centering
    \begin{tabular}{lccc}
        \toprule
        Dataset & Gaussian & Uniform & GMM \\
        \midrule
        Burgers & 6.84 & 9.17 & 6.73 \\
        NS      & 1.76 & 2.59 & 2.15 \\
        \bottomrule
    \end{tabular}
\caption{Relative $\ell_1$ error (\%) under different base distributions.}
\label{tab:base_distribution}
\end{table}

\subsection{Robustness to Noisy Sparse Observations}

We also evaluate PerFlow under noisy sparse observations by adding zero-mean Gaussian noise to the observed entries.
The noise scale is set relative to the standard deviation of the observed values.
As shown in Table~\ref{tab:noise_sparse_obs}, the reconstruction error increases smoothly as the noise level grows, indicating that PerFlow remains stable under moderate observation noise.
\begin{table}[h]
    \centering
    \begin{tabular}{lccc}
        \toprule
        Dataset & 1\% & 5\% & 10\% \\
        \midrule
        Poisson & 4.72 & 5.70 & 7.39 \\
        Burgers & 8.37 & 8.67 & 9.28 \\
        \bottomrule
    \end{tabular}
    \caption{Relative $\ell_1$ error (\%) under noisy sparse observations.}
    \label{tab:noise_sparse_obs}
\end{table}

\section{Supplementary Reconstruction Results with Long-run Baselines}
\label{sec:supp_main_results}

\begin{sidewaysfigure*}[t]
    \centering
    \includegraphics[width=\textheight]{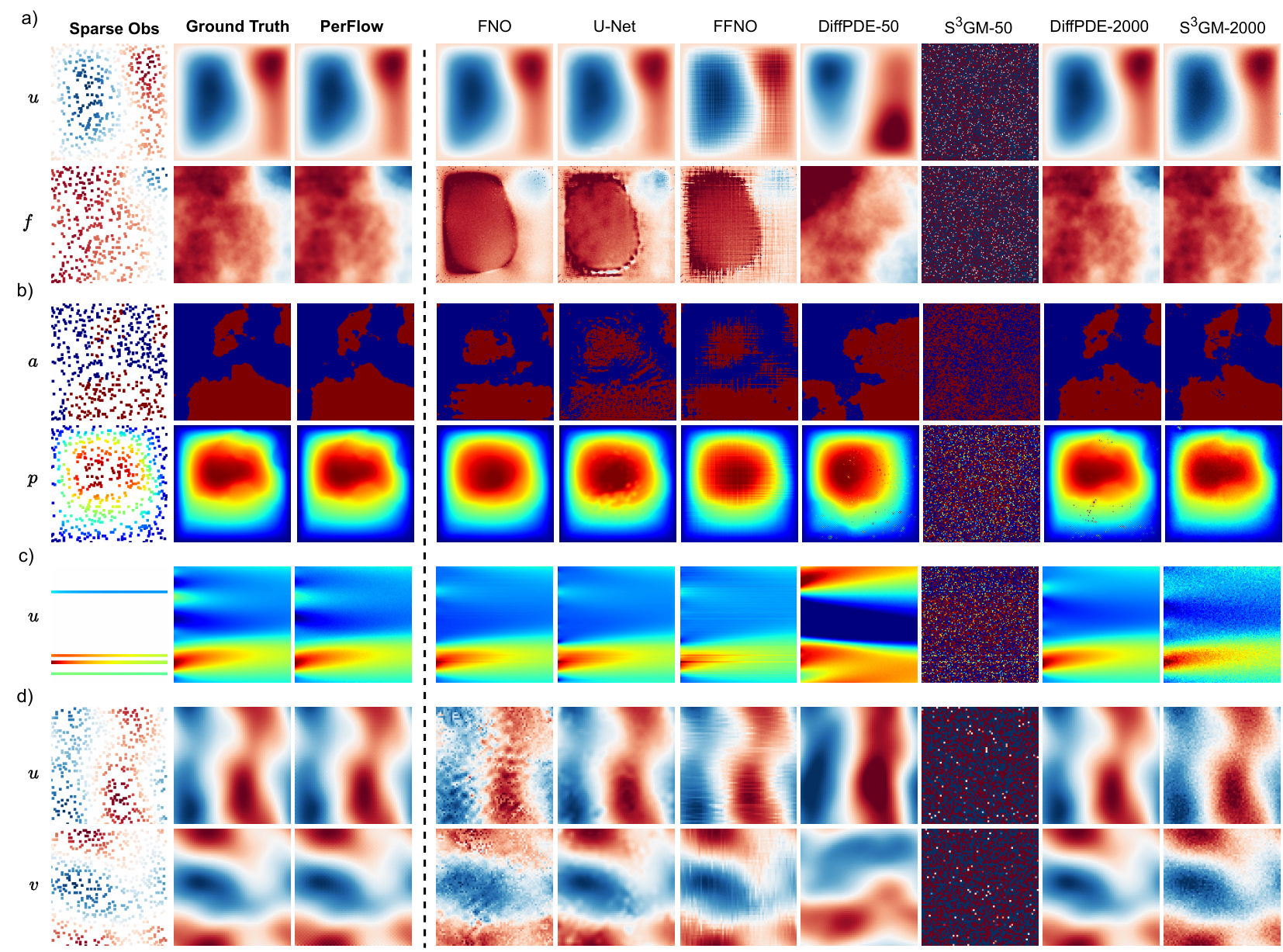}
    \caption{Sparse reconstruction on four PDE benchmarks: \textbf{(a)} Poisson (solution and forcing), \textbf{(b)} Darcy (coefficient and pressure), \textbf{(c)} Burgers, and \textbf{(d)} NS (velocity components $u$/$v$). We additionally include long-run guided sampling results for DiffPDE and S$^3$GM with 2000 steps (as used in the original papers).}
    \label{fig:main_appendix}
\end{sidewaysfigure*}

The main paper reports comparisons under a practical few-step budget (50 steps) for all generative methods.
For completeness, Fig.~\ref{fig:main_appendix} provides an enlarged version of the main qualitative results by additionally including DiffPDE-2000 and S$^3$GM-2000 (using the sampling step counts from the original papers).
This figure enables a direct visual comparison between PerFlow (50 steps) and long-run guided diffusion baselines under the same sparse observations.

\section{More Results under Varying Numbers of Observations}
\label{sec:more_obs_results}

\begin{sidewaysfigure*}[t]
    \centering
    \includegraphics[width=\textheight]{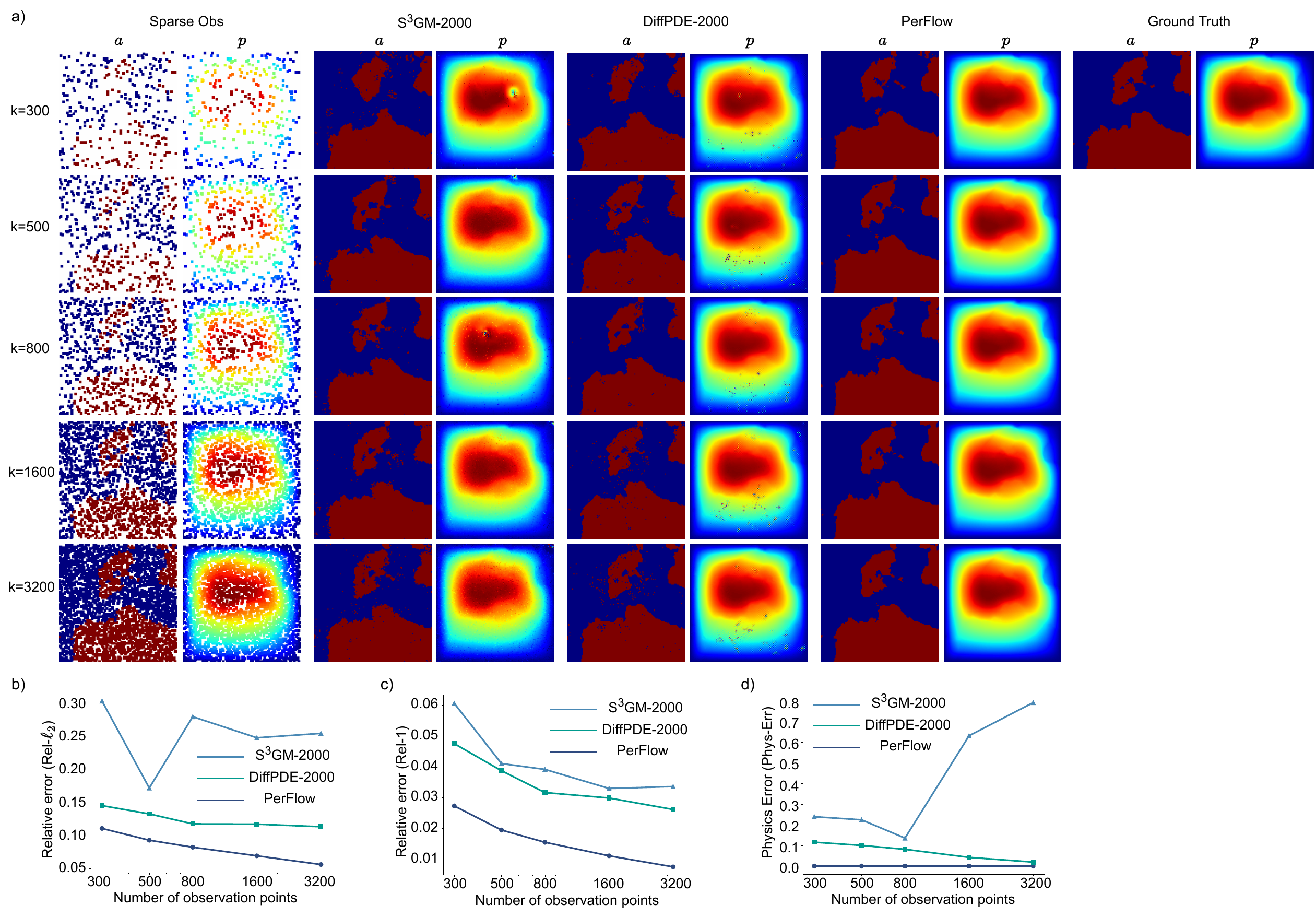}
    \caption{Darcy flow reconstruction under varying numbers of observations ($k\in\{300,500,800,1600,3200\}$) on a randomly selected test sample, comparing DiffPDE-2000, S$^3$GM-2000, and PerFlow.}
    \label{fig:obs}
\end{sidewaysfigure*}

Fig.~\ref{fig:obs} studies how reconstruction quality scales with the number of available measurements on the Darcy benchmark.
We randomly select one test sample and vary the number of observed points $k\in\{300,500,800,1600,3200\}$ (i.e., increasing observation ratios), while keeping all other settings unchanged.
We compare PerFlow against the guided generative baselines DiffPDE and S$^3$GM, both evaluated with 2000 sampling steps (as used in the original papers).
The figure includes a qualitative comparison between sparse observations, reconstructions, and the ground truth, as well as quantitative curves of Rel-$\ell_2$, Rel-$\ell_1$, and Phys-Err versus $k$.

As expected, all methods improve as more observations are provided.
However, PerFlow consistently achieves lower Rel-$\ell_2$ and Rel-$\ell_1$ errors across the full range of $k$, and simultaneously yields smaller Phys-Err, indicating better constraint satisfaction under sparse conditioning.
Notably, in the low-observation regime (e.g., $k\le 500$), guided diffusion baselines tend to exhibit larger errors and visible artifacts in unobserved regions, whereas PerFlow produces more coherent reconstructions that better preserve global structure.
With denser observations (e.g., $k\ge 1600$), the performance gap narrows, suggesting that the inverse problem becomes less ill-posed when sufficient measurements are available.
Overall, this experiment further supports that PerFlow provides a favorable accuracy--physics trade-off, particularly in the challenging sparse-measurement setting.

\end{document}